\documentclass[]{article}
\usepackage{proceed2e}

\usepackage{times}

\usepackage{amsfonts}
\usepackage{amsmath}
\usepackage{amssymb}
\usepackage{amsthm}
\usepackage{graphicx}
\usepackage{hyperref}
\usepackage{mathtools}
\usepackage{units}
\usepackage{url}

\usepackage{caption}
\usepackage{subcaption}

\newtheorem{thm}{Theorem}
\newtheorem{lem}[thm]{Lemma}
\DeclareMathOperator*{\argmin}{arg\,min}
\DeclareMathOperator{\DKL}{D_{KL}}
\DeclareMathOperator{\E}{E}
\newcommand{\Bb}{ {\bf B} }
\newcommand{\C}[1]{\mathcal{#1}}
\newcommand{\eq}[1]{\begin{align}#1\end{align}}
\newcommand{\nn}{\nonumber}

\renewcommand{\footnotemark}{}

\graphicspath{{figures/}}

\title{Taming the Noise in \\ Reinforcement  Learning via Soft Updates}

\author{ {\bf Roy Fox$^*$\thanks{$^*$These authors contributed equally to this work.}} \\
Hebrew University \\
\And
{\bf Ari Pakman$^*$}  \\
Columbia University \\
\And
{\bf Naftali Tishby}   \\
Hebrew University\\
}

\begin{document}

\maketitle

\begin{abstract}
Model-free reinforcement learning algorithms, such as Q-learning, perform poorly in the early stages of learning in noisy environments, because much effort is spent unlearning biased estimates of the state-action value function. The bias results from selecting, among several noisy estimates, the apparent optimum, which may actually be suboptimal. We propose G-learning, a new off-policy learning algorithm that regularizes the value estimates by penalizing deterministic policies in the beginning of the learning process. We show that this method reduces the bias of the value-function estimation, leading to faster convergence to the optimal value and the optimal policy. Moreover, G-learning enables the natural incorporation of prior domain knowledge, when available. The stochastic nature of G-learning also makes it avoid some exploration costs, a property usually attributed only to on-policy algorithms. We illustrate these ideas in several examples, where G-learning results in significant improvements of the convergence rate and the cost of the learning process.
\end{abstract}

\section{INTRODUCTION}

The need to separate signals from noise stands at the center of any learning task in a noisy environment.
While a rich set of tools to regularize learned parameters has been developed for supervised and unsupervised learning problems,
in areas such as reinforcement learning there still exists a vital need for techniques that tame the noise and avoid overfitting and local minima.

One of the central algorithms in reinforcement learning is Q-learning~\cite{watkins1992q},
a model-free off-policy algorithm, which  attempts to estimate the optimal value function $Q$, the cost-to-go of the optimal policy.
To enable this estimation, a stochastic exploration policy is used by the learning agent to interact with its environment and explore the  model. 
This approach is very successful and popular, and despite several alternative approaches developed in recent years~\cite{sutton1998reinforcement,powell2007approximate, szepesvari2010algorithms}, 
it is still being applied successfully in complex domains for which explicit models are lacking~\cite{mnih2015human}.

However, in noisy domains, in early stages of the learning process, the min (or max) operator in Q-learning 
brings about a bias in the estimates. This problem is akin to the ``winner's curse" in auctions~\cite{capen1971competitive,thaler1988anomalies, van2004rational, smith2006optimizer}.
With too little evidence, the biased estimates may lead to wrong decisions, 
which slow down the convergence of the learning process, and require subsequent unlearning of these suboptimal behaviors.

In this paper we present G-learning, a new off-policy information-theoretic approach to regularizing 
the state-action value function 
learned by an agent interacting with its environment in model-free settings.

This is achieved by adding to  the cost-to-go a  term that penalizes deterministic policies
which diverge from a simple stochastic prior policy~\cite{rubin2012trading}.
With only a small sample to go by, G-learning prefers a more randomized policy, and as samples accumulate, 
it gradually shifts to a more deterministic and exploiting policy.
This transition is managed by appropriately scheduling the coefficient of the  penalty term as learning proceeds.

In Section~\ref{sec:beta} we discuss the theoretical and practical aspects of scheduling this coefficient, and suggest that a simple linear schedule can perform well.
We show that G-learning with this schedule reduces the value estimation bias by avoiding overfitting in its selection of the update policy.
We further establish empirically the link between bias reduction and learning performance, that has been the underlying assumption in many approaches to reinforcement learning~\cite{hasselt2010double,ghavamzadeh2011speedy,lee2012intelligent,deep2q}.
The examples in Section~\ref{sec:examples} demonstrate the significant improvement thus obtained. 

Furthermore, in domains where exploration incurs significantly higher costs than exploitation, such as the classic cliff domain~\cite{sutton1998reinforcement}, 
G-learning  with an $\epsilon$-greedy exploration policy is exploration-aware, and chooses a less costly exploration policy, thus reducing the costs incurred during the learning process. 
Such awareness to the cost of exploration is usually attributed to on-policy algorithms, such as SARSA~\cite{sutton1998reinforcement,szepesvari2010algorithms} and Expected-SARSA~\cite{van2009theoretical,john1994best}.
The remarkable finding that G-learning exhibits on-policy-like properties is illustrated in the example of Section~\ref{sec:cliff}.

In Section~\ref{sec:noisy} we discuss the problem of learning in noisy environments.
In Section~\ref{sec:learn} we introduce the penalty term, derive G-learning and prove its convergence.
In Section~\ref{sec:beta} we determine a schedule for the coefficient of the information penalty term. 
In Section~\ref{sec:related} we discuss related work. 
In Section~\ref{sec:examples} we illustrate the strengths of the algorithm through several examples.

\section{LEARNING IN NOISY ENVIRONMENTS}
\label{sec:noisy}

\subsection{NOTATION AND BACKGROUND}

We consider the usual setting of a Markov Decision Process (MDP), in which an agent interacts with its environment by repeatedly observing its state $s\in S$, taking an action $a\in A$, with $A$ and $S$ finite, and incurring cost $c\in\mathbb R$.
This induces a stochastic process $s_0,a_0,c_0,s_1,\ldots$, where $s_0$ is fixed, and where for $t\ge0$ we have the Markov properties indicated by the conditional distributions
$a_t\sim\pi_t(a_t|s_t)$,
$c_t\sim\theta(c_t|s_t,a_t)$ and
$s_{t+1}\sim p(s_{t+1}|s_t,a_t)$.

The objective of the agent is to find a time-invariant policy~$\pi$ that minimizes the total discounted expected cost
\eq{\label{eq:vdef}
V^\pi(s)=\sum_{t\ge0}\gamma^t\E[c_t|s_0=s],
}
simultaneously for any $s\in S$, for a given discount factor $0\le\gamma<1$.
For each $t$, the expectation above is over all trajectories of length $t$ starting at $s_0=s$.
A related quantity is the state-action value function
\eq{\nn
Q^\pi(s,a)&=\sum_{t\ge0}\gamma^t\E[c_t|s_0=s,a_0=a]
\\
\label{eq:qdef}
&=\E_\theta[c|s,a]+\gamma\E_p[V^\pi(s')|s,a],
}
which equals the total discounted expected cost that follows from choosing action $a$ in state $s$, and then following the policy $\pi$.

If we know the distributions $p$ and $\theta$ (or at least $\E_\theta[c|s,a]$), then it is easy to find the optimal state-action value function
\eq{
Q^*(s,a)=\min_\pi Q^\pi(s,a)
}
using standard techniques, such as Value Iteration~\cite{bertsekas1995dynamic}.
Our interest is in model-free learning, where the model parameters are unknown.
Instead, the agent obtains samples from $p(s_{t+1}|s_t,a_t)$ and $\theta(c_t|s_t,a_t)$ through its interaction with the environment.
In this setting, the Q-learning algorithm~\cite{watkins1992q} provides a method for estimating $Q^*$.
It starts with an arbitrary $Q$, and in step $t$ upon observing $s_t$, $a_t$, $c_t$ and $s_{t+1}$, performs the update
\eq{\label{eq:qlearn}
Q(s_t,a_t) \gets{}&(1-\alpha_t)Q(s_t,a_t)
\\
\nn
&+ \alpha_t\left(c_t+\gamma\sum_{a'}\pi(a'|s_{t+1})Q(s_{t+1},a')\right),
}
with some learning rate $0\le\alpha_t\le1$, and the greedy policy for $Q$ having
\eq{\label{eq:opt_pi}
\pi(a|s)=\delta_{a,a^*(s)};&&a^*(s)=\argmin_aQ(s,a) .
}
$Q(s,a)$ is unchanged for any $(s,a)\neq(s_t,a_t)$.
If the learning rate satisfies
\eq{
\sum_t\alpha_t=\infty;&&\sum_t\alpha_t^2<\infty,
\label{alphas}
}
and the interaction itself uses an exploration policy that returns to each state-action pair infinitely many times, then $Q$ is a consistent estimator, 
converging to $Q^*$ with probability~1~\cite{watkins1992q,bertsekas1995dynamic}. 
Similarly, if the update rule~\eqref{eq:qlearn} uses a fixed update policy $\pi=\rho$, we call this algorithm Q$^\rho$-learning, because $Q$ converges to $Q^\rho$ with probability 1.

\subsection{BIAS AND EARLY COMMITMENT}\label{sec:bias}

Despite the success of Q-learning in many situations, learning can proceed extremely slowly when there is  noise 
in the distribution, given $s_t$ and $a_t$, of either of the terms of~\eqref{eq:qdef}, namely the cost $c_t$ and the value of the next state $s_{t+1}$.
The source of this problem is a negative bias introduced by the min operator in the estimator $\min_{a'}Q(s_{t+1},a')$, when~\eqref{eq:opt_pi} is plugged into~\eqref{eq:qlearn}.

To illustrate this bias, assume that $Q(s,a)$ is an unbiased but noisy estimate of the optimal $Q^*(s,a)$.
Then Jensen's inequality for the concave min operator implies that
\eq{ 
\E[\min_{a}Q(s,a)] \le \min_{a} Q^*(s,a) ,
\label{bias}
}
with equality only when $Q$ already reveals the optimal policy by having $\argmin_aQ(s,a)=\argmin_aQ^*(s,a)$ with probability 1, so that no further learning is needed.
The expectation in~\eqref{bias} is with respect to the learning process, including any randomness in state transition, cost, exploration and internal update, given the domain.

This is an optimistic bias, causing the cost-to-go to appear lower than it is (or the reward-to-go higher).
It is the well known ``winner's curse" problem in economics and decision theory~\cite{capen1971competitive,thaler1988anomalies, van2004rational, smith2006optimizer},
and in the context of Q-learning it was studied before in~\cite{powell2007approximate, hasselt2010double,ghavamzadeh2011speedy, lee2012intelligent}.
A similar problem occurs when a function approximation scheme is used for $Q$ instead of a table, even in the absence of transition or cost noise, because 
the approximation itself introduces noise~\cite{thrun1993issues}.

As the sample size increases, the variance in $Q(s,a)$ decreases, which in turn reduces the bias in~\eqref{bias}.
This makes the update policy~\eqref{eq:opt_pi} more optimal, and the update increasingly similar to Value Iteration.

\subsection{THE INTERPLAY OF VALUE BIAS AND POLICY SUBOPTIMALITY}

It is insightful to consider the effect of the bias not only on the estimated value function, but also on the real value $V^\pi$ of the greedy policy~\eqref{eq:opt_pi},
since in many cases the latter is the actual output of the learning process.
The central quantity of interest here is the gap $Q^*(s,a')-V^*(s)$, in a given state $s$, between the value of a non-optimal action $a'$ and that of the optimal action.

Consider first the case in which the gap is large compared to the noise in the estimation of the $Q(s,a)$ values.  
In this case, $a'$ indeed appears suboptimal with high probability, as desired. 
Interestingly, when the gap is very small relative to the noise, the learning agent should not worry, either. 
Confusing such $a'$ for the optimal action has a limited effect on the value of the greedy policy, since choosing $a'$ is near-optimal.

We conclude 
that the real value $V^\pi$ of the greedy policy~\eqref{eq:opt_pi} is suboptimal only in the intermediate regime, when the gap is of the order of the noise, 
and neither is small. 
The effect of the noise can be made even worse by the propagation of bias between states, through updates.
Such propagation can cause large-gap suboptimal actions to nevertheless appear optimal, if they lead to a region of state-space that is highly biased.

\subsection{A DYNAMIC OPTIMISM-UNCERTAINTY LOOP}
The above considerations were agnostic to the exploration policy, but the bias reduction 
can be accelerated by an exploration policy that is close to being greedy.
In this case, high-variance estimation is self-correcting: an estimated state value with optimistic bias draws exploration towards that state, leading to a decrease in the variance,
which in turn reduces the optimistic bias. 
This is a dynamic form of optimism under uncertainty. While in the usual case the optimism is externally imposed as an initial condition~\cite{brafman2003r},
here it is spontaneously generated by the noise and self-corrected through exploration.

The approach we propose below to reduce the variance is motivated by electing to represent the uncertainty explicitly, and not indirectly through an optimistic bias.
We notice that although \emph{in the end} of the learning process one obtains the deterministic greedy policy from $Q(a,s)$ as in~\eqref{eq:opt_pi}, 
\emph{during} the learning itself the bias in $Q$ can be ameliorated by avoiding the hard min operator, and refraining from committing to a deterministic greedy policy.
This can be achieved by adding to $Q$, at the early learning stage, a term that penalizes deterministic policies, which we consider next.


\section{LEARNING WITH SOFT UPDATES}
\label{sec:learn}

\subsection{THE FREE-ENERGY FUNCTION $G$ AND G-LEARNING}

Let us adopt, before any interaction with the environment, a simple stochastic prior policy $\rho(a|s)$.
For example, we can take the uniform distribution over the possible actions.
The \emph{information cost} of a learned policy $\pi(a|s)$ is defined as
\eq{
g^\pi(s,a)=\log\tfrac{\pi(a|s)}{\rho(a|s)},
\label{info_cost}
}
and its expectation over the policy $\pi$ is the Kullback-Leibler (KL) divergence of $\pi_s=\pi(\cdot|s)$ from $\rho_s=\rho(\cdot|s)$, 
\eq{ 
\E_\pi[g^\pi(s,a)|s]=\DKL[\pi_s\|\rho_s].
}
The term~\eqref{info_cost} penalizes deviations from the prior policy and serves to regularize the optimal policy away from a deterministic action.
In the context of the MDP dynamics $p(s_{t+1}|s_t, a_t)$, similarly to~\eqref{eq:vdef}, we consider the total discounted expected information cost
\eq{\label{eq:info}
I^\pi(s)=\sum_{t\ge0}\gamma^t\E[g^\pi(s_t,a_t)|s_0=s].
}
The discounting in~\eqref{eq:vdef} and~\eqref{eq:info} is justified by imagining a horizon $T\sim\text{Geom}(1-\gamma)$, distributed geometrically with parameter $1-\gamma$.
Then the cost-to-go $V^\pi$ in~\eqref{eq:vdef} and the information-to-go $I^\pi$ in~\eqref{eq:info} are the total (undiscounted) expected $T$-step costs.

Adding the penalty term \eqref{eq:info} to the cost function~\eqref{eq:vdef} gives 
\eq{
F^{\pi}(s) &= V^\pi(s)+\tfrac1\beta I^\pi(s) ,
\label{fdef}
\\
\nn
&= \sum_{t\ge0}\gamma^t\E[\tfrac1\beta g^\pi(s_t,a_t)+c_t|s_0=s],
}
called the \emph{free-energy function} by analogy with a similar quantity in statistical mechanics~\cite{rubin2012trading}.

Here $\beta$ is a parameter that sets the relative weight between the two costs.
For the moment, we  assume that $\beta$ is fixed.
In following sections, we let $\beta$  grow as the learning proceeds.

In analogy with the $Q^{\pi}$ function~\eqref{eq:qdef}, let us define the \emph{state-action free-energy function} $G^\pi(s,a)$ as
\eq{
\MoveEqLeft G^{\pi}(s,a)=\E_\theta[c|s,a]+\gamma\E_p[F^\pi(s')|s,a]
\label{gdef}
\\
\nn
&= \sum_{t\ge0}\gamma^t \E[c_t+\tfrac\gamma\beta g^\pi(s_{t+1},a_{t+1}) ) | s_0=s,a_0=a],
}
and note that it does not involve the information term at time $t=0$, since the action $a_0=a$ is already known.
From the definitions~\eqref{fdef} and~\eqref{gdef}  it follows that
\eq{ 
F^{\pi}(s) = \sum_a \pi(a|s)  \left[ \tfrac{1}{\beta} \log \tfrac{\pi(a|s) }{\rho(a|s)} + G^{\pi}(s,a) \right].
\label{eq:opt}
} 

It is easy to verify that, given the $G$ function, the above expression for $F^{\pi}$ has gradient 0 at
\eq{
\pi(a|s) = \frac{ \rho(a|s)  e^{-\beta G(s,a)} }  {\sum_{a'} \rho(a'|s) e^{-\beta G(s,a')}   }    ,
\label{pp}
}
which is therefore the optimal policy.


The policy~\eqref{pp} is the soft-min operator applied to $G$, with inverse-temperature $\beta$.
When $\beta$ is small, the information cost is dominant, and $\pi$ approaches the prior $\rho$.
When $\beta$ is large, we are willing to diverge much from the prior to reduce the external cost, and $\pi$ approaches the deterministic greedy policy for $G$.

Evaluated at the soft-greedy policy~\eqref{pp}, the free energy~\eqref{eq:opt} is
\eq{
F^\pi(s)=-\tfrac1\beta\log\sum_{a}\rho(a|s)e^{-\beta G^\pi(s,a)},
}
and plugging this expression into~\eqref{gdef}, we get that the optimal $G^*$ 
is a fixed point of the equation
\eq{
G^*(s,a) ={}&\E_\theta[c|s,a]
\label{eq:hg}
 \\
 \nn
 &-\tfrac\gamma\beta\E_p\left[\log\sum_{a'}\rho(a'|s')e^{-\beta G^*(s',a')}\right]
\\
\equiv{}& \Bb^*[G^*]_{(s,a)}.
\label{eq:hq}
}
Based on the above expression, we introduce G-learning as 
an off-policy TD-learning algorithm~\cite{sutton1998reinforcement}, that learns the optimal $G^*$ from the interaction with the environment by applying the update rule
\eq{
\MoveEqLeft G(s_t,a_t) \gets  (1-\alpha_t)G(s_t,a_t) 
\label{glearn}
\\
\nn
& + \alpha_t \left( c_t - \tfrac{\gamma}{\beta} \log \left( \sum_{a'} \rho(a'|s_{t+1}) e^{-\beta G (s_{t+1},a')}  \right) \right).
}


\subsection{THE ROLE OF THE PRIOR}
Clearly the choice of the prior policy $\rho$ is significant in the performance of the algorithm.
The prior policy can encode any prior knowledge that we have about the domain, and this can improve the convergence if done correctly.
However an incorrect prior policy can hinder learning.
We should therefore choose a prior policy that represents all of our prior knowledge, but nothing more.
This prior policy has maximal entropy given the prior knowledge~\cite{jaynes2003probability}.

In our examples in Section~\ref{sec:examples}, we use the uniform prior policy, representing no prior knowledge.
Both in Q-learning and in G-learning, we could utilize the prior knowledge that moving into a wall is never a good action, by eliminating those actions.
One advantage of G-learning is that it can utilize softer prior knowledge.
For example, a prior policy that gives lower probability for moving into a wall represent the prior knowledge that such an action is usually (but not always) harmful, a type of knowledge that cannot be utilized in Q-learning.

We have presented G-learning in a fully parameterized formulation, where the function $G$ is stored in a lookup table.
Practical applications of Q-learning often resort to approximating the function $Q$ through function approximations, 
such as linear expansions or  neural networks~\cite{sutton1998reinforcement,powell2007approximate,szepesvari2010algorithms, busoniu2010reinforcement,mnih2015human}.
Such an approximation generates inductive bias, which is another form of implicit prior knowledge.
While G-learning is introduced here in its table form, preliminary results indicate that its benefits carry over to function approximations, despite the challenges posed by this extension.

\subsection{CONVERGENCE}

In this section we study the convergence of  $G$ under the update rule~\eqref{glearn}.
Recall that the supremum norm is defined as $| x |_{\infty} = \max_{i}|x_i|$.
We need the following Lemma, proved in Appendix~\ref{sec:apx}. 
\begin{lem}
The operator $\Bb^*[G]_{(s,a)}$  defined in~\eqref{eq:hq} is a contraction in the supremum norm, 
\eq{
 \big| \Bb^*[G_1] - \Bb^*[G_2]  \big|_{\infty} \leq  \gamma \big| G_1 - G_2 \big|_{\infty}.
}

\end{lem}

The update equation~\eqref{glearn} of the algorithm can be written as a stochastic iteration equation
\eq{ 
G_{t+1}(s_{t},a_{t}) ={}& (1-\alpha_t)G_t(s_t,a_t)
\label{iter}
\\
&+ \alpha_t ( \Bb^*[G_{t}]_{(s_t, a_t)} + z_t(c_t, s_{t+1}) )
\nn
} 
where the random variable $z_t$ is
\eq{
z_t(c_t, s_{t+1}) ={}& -\Bb^*[G_{t}]_{(s_t, a_t)}
\\
& + c_t - \tfrac{\gamma}{\beta} \log  \sum_{a'} \rho(a'|s_{t+1}) e^{-\beta G_t (s_{t+1},a') } .
\nn
} 
Note that $z_t$ has expectation 0.  Many results exist for iterative equations of the type~\eqref{iter}.
In particular, given conditions~\eqref{alphas} for $\alpha_t$, the contractive nature of $\Bb^*$,
infinite  visits to each pair $(s_t,a_t)$ and assuming that $|z_t|< \infty$ , 
$G_t$ is guaranteed to converge to the optimal $G^*$ with probability~1~\cite{bertsekas1995dynamic, borkar2008stochastic}.

\section{SCHEDULING \boldmath$\beta$}\label{sec:beta}

In the previous section, we showed that running G-learning with a fixed $\beta$ converges, with probability 1, to the optimal $G^*$ for that $\beta$, given by the recursion in~\eqref{gdef}--\eqref{pp}.
When $\beta=\infty$, the equations for $G^*$ and $F^*$ degenerate into the equations for $Q^*$ and $V^*$, and G-learning becomes Q-learning.
When $\beta=0$, the update policy $\pi$ in~\eqref{pp} is equal to the prior $\rho$.
This case, denoted Q$^\rho$-learning, converges to $Q^\rho$.

In an early stage of learning, Q$^\rho$-learning has an advantage over Q-learning, because it avoids committing to a deterministic policy based on a noisy $Q$ function.
In a later stage of learning, when $Q$ is a more precise estimate of $Q^*$, Q-learning gains the advantage by updating with a better policy than the prior.
This is demonstrated in section~\ref{sec:grid}.

We would therefore like to schedule $\beta$ so that G-learning makes a smooth transition from Q$^\rho$-learning to Q-learning, 
just at the right pace to enjoy the early advantage of the former and the late advantage of the latter.
As we argue below, such a $\beta$ always exists. 

\subsection{ORACLE SCHEDULING}

To consider the effect of the $\beta$ scheduling on the correction of the bias~\eqref{bias}, suppose that during learning we reach some $G$ that is an unbiased estimate of $G^*$.
$G(s_t,a_t)$ would remain unbiased if we update it towards
\eq{\label{eq:unbiased}
c_t+\gamma G(s_{t+1},a^*)
}
with
\eq{
a^*=\argmin_{a'}G^*(s_{t+1},a'),
}
but we do not have access to this optimal action.
If we use the update rule~\eqref{glearn} with $\beta=0$, we update $G(s_t,a_t)$ towards
\eq{
c_t+\gamma\sum_{a'}\rho(a'|s_{t+1})G(s_{t+1},a'),
}
which is always at least as large as~\eqref{eq:unbiased}, creating a positive bias.
If we use $\beta=\infty$, we update $G(s_t,a_t)$ towards
\eq{
c_t+\gamma\min_{a'}G(s_{t+1},a'),
}
which creates a negative bias, as explained in Section~\ref{sec:bias}.
Since the right-hand side of~\eqref{glearn} is continuous and monotonic in $\beta$, there must be some $\beta$ for which this update rule is unbiased.

This is a non-constructive proof for the existence of a $\beta$ schedule that keeps the value estimators unbiased (or at least does not accumulate additional bias).
We can imagine a scheduling oracle, and a protocol for the agent by which to consult the oracle and obtain the $\beta$ for its soft updates.
At the very least, the oracle must be told the iteration index $t$, but it can also be useful to let $\beta$ depend on any other aspect of the learning process, particularly the current world state $s_t$.

\subsection{PRACTICAL SCHEDULING}\label{sec:beta_practical}
A good schedule should increase $\beta$ as learning proceeds, because as more samples are gathered 
the variance of $G$ decreases, allowing more deterministic policies. 
In the examples of Section~\ref{sec:examples} we adopted the linear schedule
\eq{\label{eq:sched}
\beta_t=kt,
}
with some constant $k>0$. 
Another possibility that we explored was to make $\beta$ inversely proportional to a running average of the Bellman error, which decreases as learning progresses. 
The results were similar to the linear schedule. 



The optimal parameter $k$ can be obtained by performing initial runs with different values of $k$ and picking the value whose learned policy gives empirically the lower cost-to-go. 
Although this exploration would seem costly compared to other algorithms for which no parameter tuning is needed, these initial runs do not need to be carried for many iterations.
Moreover, in many situations the agent is confronted with a class of similar domains, and  tuning $k$ in a few initial domains 
leads to an improved learning for the whole class. 
This is the case in the domain-generator example in Section~\ref{sec:grid}.

\section{RELATED WORK}
\label{sec:related}

The connection between domain noise or function approximation, and the statistical bias in the $Q$ function, was first discussed in~\cite{thrun1993issues,powell2007approximate}.
An interesting modification of Q-learning to address this problem is Double-Q-learning~\cite{hasselt2010double,deep2q}, which uses two estimators for the $Q$ function to alleviate 
the bias. Other modifications of Q-learning that attempt to reduce or correct the bias are suggested in~\cite{ghavamzadeh2011speedy,lee2012intelligent}.

An early approach to Q-learning in continuous noisy domains was to learn, instead of the value function, the advantage 
function $A(s,a) = Q(s,a)-V(s)$~\cite{baird1994reinforcement}.  The algorithm represents $A$ and $V$ separately, 
and the optimal action is determined from $A(s,a)$ as $a^*(s)=\argmin_aA(s,a)$. In noisy environments, learning $A$ 
is shown in some examples to be faster than learning  $Q$~\cite{baird1994reinforcement,baird1995advantage}. 

More recently, it was shown that the advantage learning algorithm is a gap-increasing operator~\cite{bellemare2016increasing}.
As discussed in Section~\ref{sec:bias}, the action gap is a central factor in the generation of bias, and increasing the gap should also help reduce the bias.
In Section~\ref{sec:grid} we compare our algorithm to the consistent Bellman operator $\C T_C$, one of the gap-increasing algorithms introduced in~\cite{bellemare2016increasing}.

For other works that study the effect of noise in Q-learning, although without identifying the bias~\eqref{bias}, see~\cite{pendrith1994reinforcement, pendrith1997estimator, moreno2006noisy}.

Information considerations have received attention in recent years in various machine learning settings,
with the free energy $F^\pi$ and similar quantities used as a design principle for policies in known MDPs~\cite{rubin2012trading, todorov2006linearly, kappen2012optimal}.
Other works have used related methods for reinforcement learning~\cite{todorov2009efficient,peters2010relative, rawlik2010approximate,azar2012dynamic,still2012information}.
A KL penalty similar to ours is used in~\cite{still2012information}, in settings with known reward and transition functions, to encourage ``curiosity".

Soft-greedy policies have been used before for exploration~\cite{sutton1998reinforcement,tokic2011value}, but to our knowledge G-learning is the first TD-learning algorithm to explicitly use soft-greedy policies in its updates.

Particularly relevant to our work is the approach studied in~\cite{peters2010relative}.
There the policy is iteratively improved by optimizing it in each iteration under the constraint that it only diverges slightly, in terms of KL-divergence, from the empirical distribution generated by the previous policy.
In contrast, in G-learning we measure the KL-divergence from a fixed prior policy, and in each iteration allow the divergence to grow larger by increasing $\beta$.
Thus the two methods follow different information-geodesics from the stochastic prior policy to more and more deterministic policies.

This distinction is best demonstrated by considering the $\Psi$-learning algorithm presented in~\cite{rawlik2010approximate, azar2012dynamic}, based on the same approach as~\cite{peters2010relative}.
It employs the update rule
\eq{\label{eq:psilearn}
\Psi(s_t,a_t) \gets{}&  \Psi(s_t,a_t) \\
& + \alpha_t ( c_t + \gamma\bar\Psi(s_{t+1}) - \bar\Psi(s_t) ),
\nn
}
with
\eq{
\bar\Psi(s) = -\log\sum_a\rho(a|s)e^{-\Psi(s,a)} , 
}
which is closely related to our update of $G$ in~\eqref{glearn}.

Apart from lacking a $\beta$ parameter, the most important difference is that the update of $\Psi$ involves subtracting $\alpha_t\bar\Psi(s_t)$, whereas the update of $G$ involves subtracting $\alpha_tG(s_t,a_t)$.
This seemingly minor modification has a large impact on the behavior of the two algorithms.
The update of $G$ is designed to pull it towards the optimal state-action free energy $G^*$, for all state-action pairs.
In contrast, subtracting the log-partition $\bar\Psi(s_t)$, in the long run pulls only $\Psi(s_t,a^*)$, with $a^*$ the optimal action, towards its true value, while for the other actions the values grow to infinity.
In this sense, the $\Psi$-learning update~\eqref{eq:psilearn} is an information-theoretic gap-increasing Bellman operator~\cite{bellemare2016increasing}.

The growth to infinity of suboptimal values separates them from the optimal value, and drives the algorithm to convergence.
In G-learning, this parallels the increase in $\beta$ with the accumulation of samples. 
However, there is a major benefit to keeping $G$ reliable in all its parameters, and controlling it with a separate $\beta$ parameter.
In $\Psi$-learning, the $\Psi$ function penalizes actions it deems suboptimal.
If early noise causes an error in this penalty, the algorithm needs to unlearn it - a similar drawback to that of Q-learning.
In Section~\ref{sec:examples}, we demonstrate the improvement offered by G-learning.

\section{EXAMPLES}
\label{sec:examples}

This section illustrates how  G-learning improves on existing model-free learning algorithms in several settings. 
The domains we use are clean and simple, to demonstrate that the advantages of G-learning are inherent to the algorithm itself.

We schedule the learning rate $\alpha_t$ as
\eq{
\alpha_t=n_t(s_t,a_t)^{-\omega}\,,
}
where $n_t(s_t,a_t)$ is the number of times the pair $(s_t,a_t)$ was visited. 
This scheme is widely used, and is consistent with~\eqref{alphas} for $\omega\in(\nicefrac12,1]$.
We choose $\omega=0.8$, which is within the range suggested in~\cite{even2004learning}.

We schedule $\beta$ linearly, as discussed in Section~\ref{sec:beta_practical}.
In each case, we start with 5 preliminary runs of G-learning with various linear coefficients, 
and pick the coefficient with the lowest empirical cost. 
This coefficient is used in the subsequent test runs, whose results are plotted in Figure~\ref{fig:res}.

In all cases, we use a uniform prior policy $\rho$, a discount factor $\gamma=0.95$, and 0 for the initial values ($Q_0=0$ in Q-learning, and similarly in the other algorithms).
Except when mentioned otherwise, we employ random exploration, where $s_t$ and $a_t$ are chosen uniformly at the beginning of each time step, independently of any previous sample.
This exploration technique is useful when comparing update rules, while controlling for the exploration process.

%
%

\subsection{GRIDWORLD}\label{sec:grid}
Our first set of examples occurs in a gridworld of $8\times8$ squares, with 
some unavailable squares occupied by walls shown in black (Figure~\ref{fig:grid}).
The lightest square is the goal, and reaching it ends the episode.

\begin{figure}[!]
\centering
\includegraphics[width=0.3\textwidth]{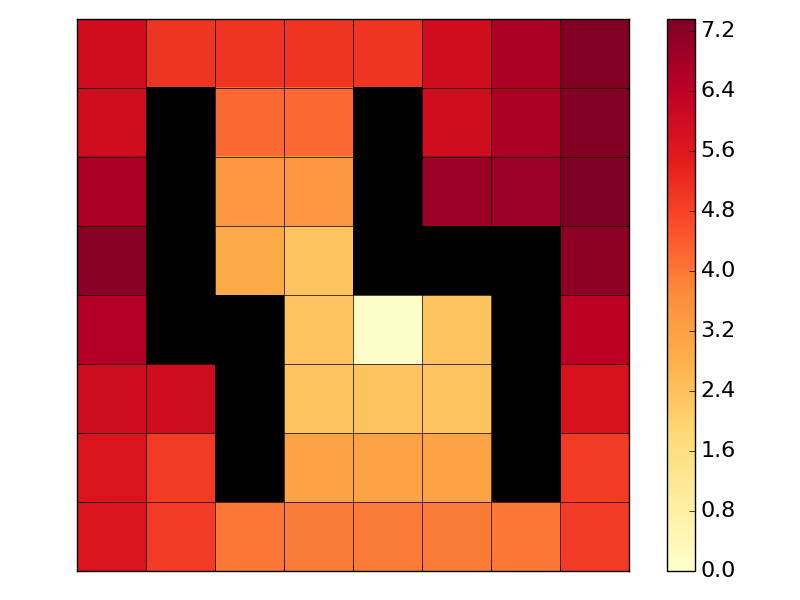}
\caption{ Gridworld domain. The agent can choose an adjacent square as the target to move to, and then may end up stochastically in a square adjacent to that target. 
The color scale indicates the optimal values $V^*$ with a fixed cost of 1 per step.}
\label{fig:grid}
\end{figure}

At each time step, the agent can choose to move one square in any of the 8 directions (including diagonally), or stay in place.
If the move is blocked by a wall or the edge of the board, it effectively attempts to stay in place.
With some probability, the action performed  by the agent is further followed by an additional random slide:
with probability $0.15$ to each vertically or horizontally adjacent available position, 
and with probability $0.05$  to each diagonally adjacent available position. 

The noise associated with these random transitions can be enhanced further by the possible variability in the costs incurred along the way. 
We consider three cases.
In the first case, the cost in each step is fixed at 1.
In the second case, the cost in each step is distributed normally i.i.d, with mean 1 and standard deviation 2.
In the third case we define a distribution over domains, such that at the time of domain-generation the mean cost for each state-action is distributed uniformly i.i.d over $[1,3]$.
Once the domain has been generated and interaction begins, the cost itself in each step is again distributed normally i.i.d, with the generated mean and standard deviation 4.

We attempt to learn these domains using various algorithms.
Figure~\ref{fig:res} summarizes the results for Q-learning, G-learning, Double-Q-learning~\cite{hasselt2010double}, $\Psi$-learning~\cite{rawlik2010approximate,azar2012dynamic} and the consistent Bellman operator $\C T_C$ of~\cite{bellemare2016increasing}.
We also include Q$^\rho$-learning, which performs updates as in~\eqref{eq:qlearn} towards the prior policy $\rho$.
Comparison with Speedy-Q-learning~\cite{ghavamzadeh2011speedy} is omitted, since it showed no improvement over vanilla Q-learning in these settings.
In our experiments, these algorithms had comparable running times.

The $\beta$ scheduling used in G-learning is linear, with the coefficient $k$ equal to $10^{-3}$, $10^{-4}$, $5\cdot10^{-5}$ and $10^{-6}$, respectively for the fixed-cost, noisy-cost, domain-generator and cliff domains (see Section~\ref{sec:cliff}).

For each case, Figure~\ref{fig:res} shows the evolution over 250,000 algorithm iterations of the following  three measures, averaged over $N=100$ runs:
\begin{enumerate}
 \item 
Empirical bias, defined as
\eq{
\tfrac{1}{Nn} \sum_{i=1}^N \sum_{s=1}^n (V_{i,t}(s)-V_i^*(s)),\label{eq:bias}
}
where $i$ indexes the $N$ runs and $s$ the $n$ states.
Here $V_{i,t}$ is the greedy value based on the estimate obtained by each algorithm ($Q$, $G$, etc.), in iteration $t$ of run $i$.
The optimal value $V_i^*$, computed via Value Iteration, varies between runs in the domain-generator case.

\item 
Mean absolute error in $V$
\eq{
\tfrac{1}{Nn} \sum_{i=1}^N \sum_{s=1}^n |V_{i,t}(s)-V_i^*(s)|.\label{eq:abserr}
}
A low bias could result from the cancellation of terms with high positive and negative biases. 
A convergence in the absolute error is more indicative of the actual convergence of the value estimates.

\item Increase in cost-to-go, relative to the optimal policy
\eq{
\tfrac{1}{Nn} \sum_{i=1}^N \sum_{s=1}^n(V^{\pi_{i,t}}(s)-V_i^*(s)).\label{eq:value}
}
This measures the quality of the learned policy. 
Here $\pi_{i,t}$ is the greedy policy based on the state-action value estimates, and $V^{\pi_{i,t}}$ is its value in the model, computed via Value Iteration. 
\end{enumerate}

An algorithm is better when these measures reach zero faster. 
As is clear in Figure~\ref{fig:res}, in  the domains with  noisy cost (Rows 2 and 3), G-learning dominates over all the other competing algorithms by the three measures. 
The results are statistically significant, but plotting confidence intervals would clutter the figure.

An important and surprising point of Figure~\ref{fig:res} is that Q$^\rho$-learning always outperforms Q-learning initially, before degrading. 
The reason is that the Q-learning updates initially rely
on very few samples, so these harmful updates need to be undone by later updates.
Q$^\rho$-learning, on the other hand, updates in the direction of a uniform prior.
This gives an early advantage in mapping out the local topology of the problem, before long-range effects start pulling the learning towards the suboptimal $Q^\rho$.

The power of G-learning is that it enjoys the early advantage of Q$^\rho$-learning, and smoothly 
transitions to the convergence advantage of Q-learning.
When $\beta$ is small, the information cost $g_t$~\eqref{info_cost} outweighs the external costs $c_t$, and we update towards $\rho$.
As samples keep coming in, and our estimates improve, $\beta$ increases, and the updates gradually lean more towards a cost-optimizing policy.
Unlike early stages in Q-learning, at this point $G_t$ is already a good estimate, and we avoid overfitting.
As mentioned above, Figure~\ref{fig:res} shows that this effect is more manifest in noisier scenarios.

Finally, Figure~\ref{fig:bellman} shows running averages of the Bellman error for the different algorithms considered. 
The Bellman error in G-learning is the coefficient multiplying $\alpha_t$ in~\eqref{glearn},
\eq{
\Delta G_t  \equiv{}& c_t - \tfrac{\gamma}{\beta} \log \left( \sum_{a'} \rho(a'|s_{t+1}) e^{-\beta G_t(s_{t+1},a')}  \right)  
\nn
\\
&- G_t(s_{t},a).
}
When learning  ends and $G=G^*$, the expectation of $\Delta G_t$ is zero (see \eqref{eq:hg}). Similar definitions hold for the other 
learning algorithms we compare with. As is clear from Figure~\ref{fig:bellman}, G-learning reaches zero average Bellman error faster than the competing methods,
even while $\beta$ is still increasing in order to make $G^*$ converge to $Q^*$.

 \makeatletter
\renewcommand{\p@subfigure}{}
\makeatother
 
 \begin{figure*}[!t]
    \centering

    \begin{subfigure}[b]{0.32\textwidth}
        \includegraphics[width=\textwidth]{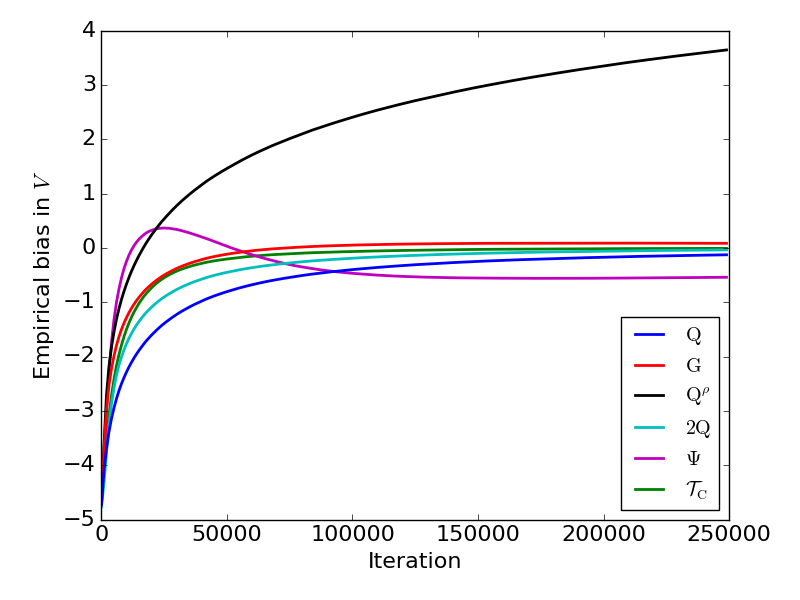}
    \end{subfigure}
    ~ 
    \begin{subfigure}[b]{0.32\textwidth}
        \includegraphics[width=\textwidth]{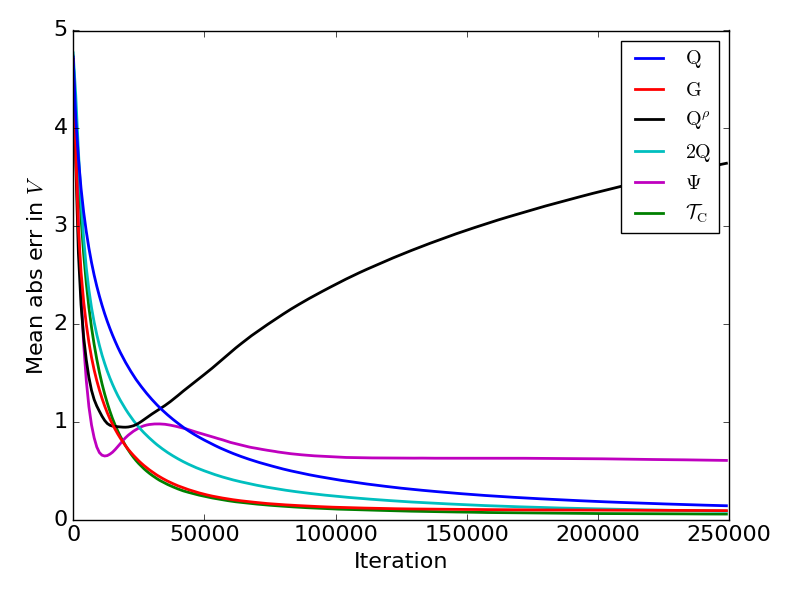}
    \end{subfigure}
    ~ 
    \begin{subfigure}[b]{0.32\textwidth}
        \includegraphics[width=\textwidth]{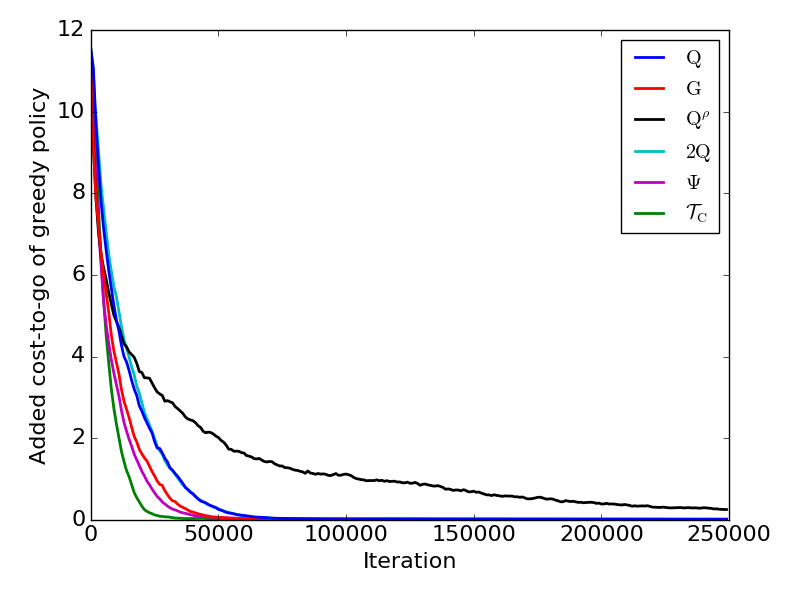}
    \end{subfigure}
\\
    \begin{subfigure}[b]{0.32\textwidth}
        \includegraphics[width=\textwidth]{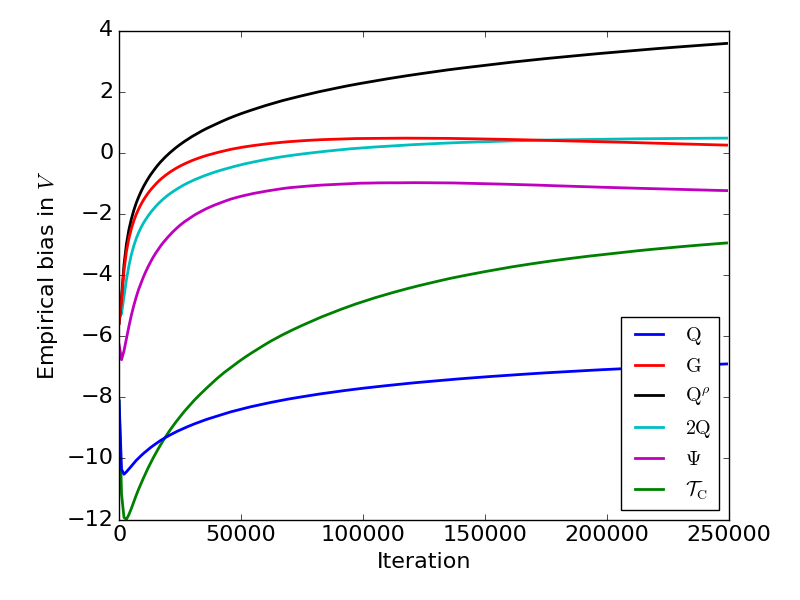}
    \end{subfigure}
    ~ 
    \begin{subfigure}[b]{0.32\textwidth}
        \includegraphics[width=\textwidth]{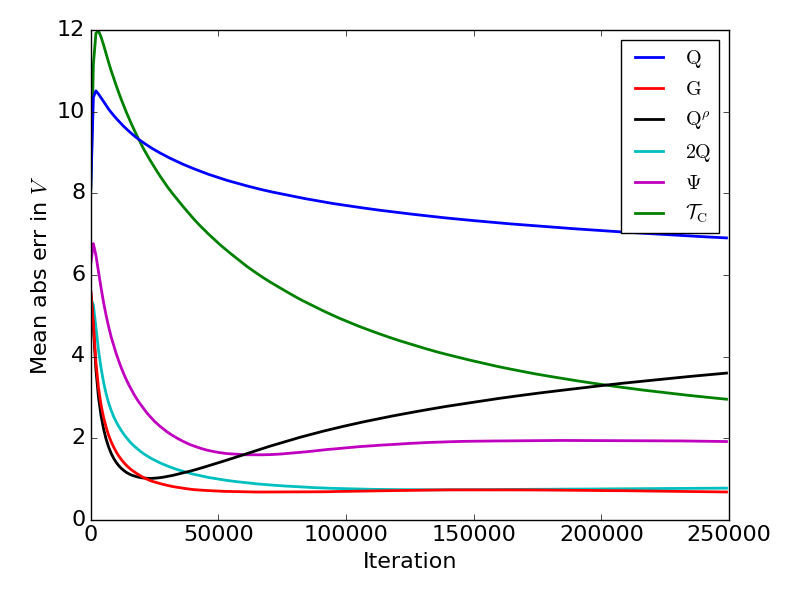}
    \end{subfigure}
    ~ 
    \begin{subfigure}[b]{0.32\textwidth}
        \includegraphics[width=\textwidth]{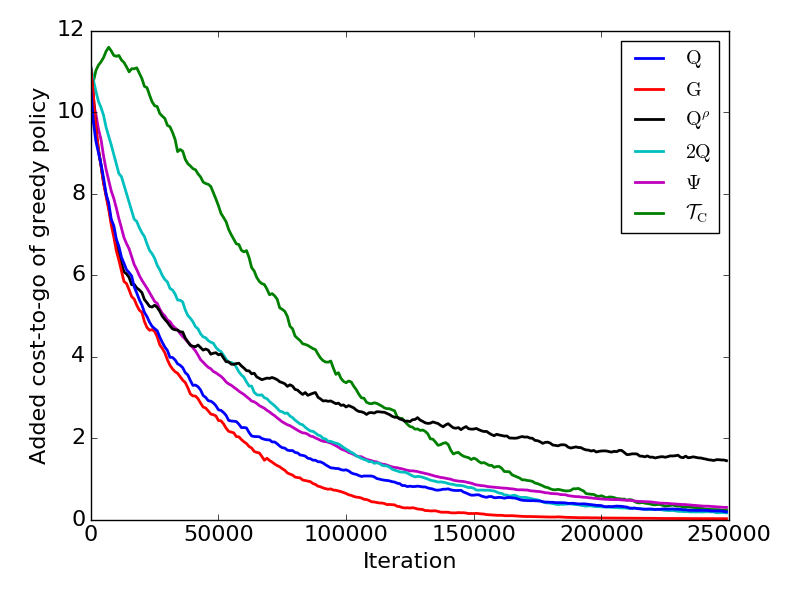}
    \end{subfigure}
\\
    \begin{subfigure}[b]{0.32\textwidth}
        \includegraphics[width=\textwidth]{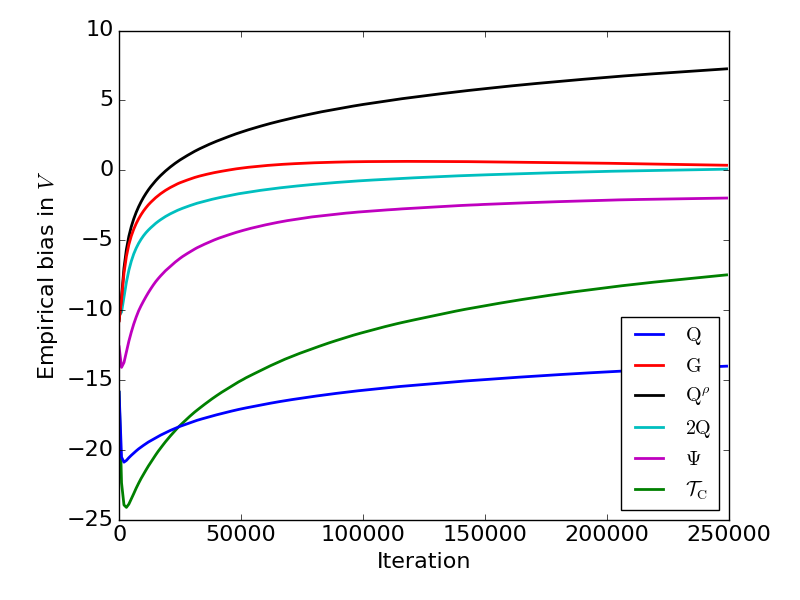}
    \end{subfigure}
    ~ 
    \begin{subfigure}[b]{0.32\textwidth}
        \includegraphics[width=\textwidth]{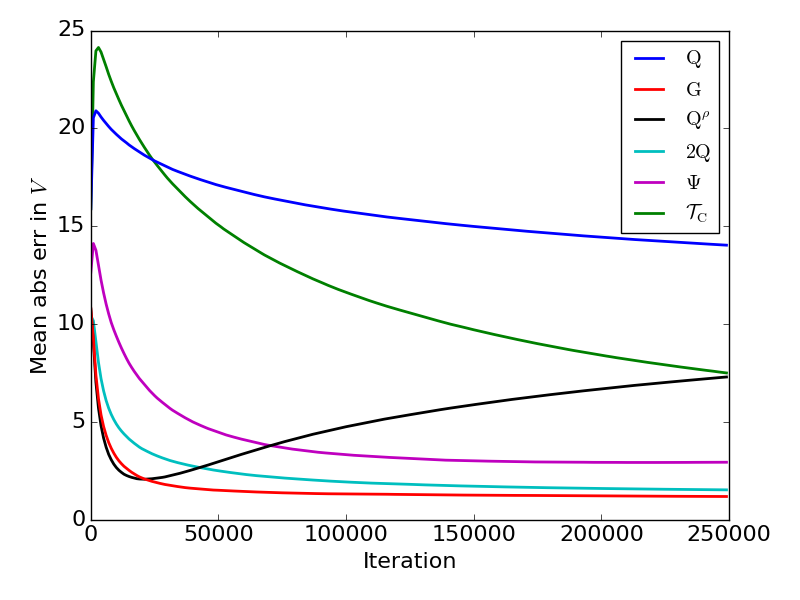}
    \end{subfigure}
    ~ 
    \begin{subfigure}[b]{0.32\textwidth}
        \includegraphics[width=\textwidth]{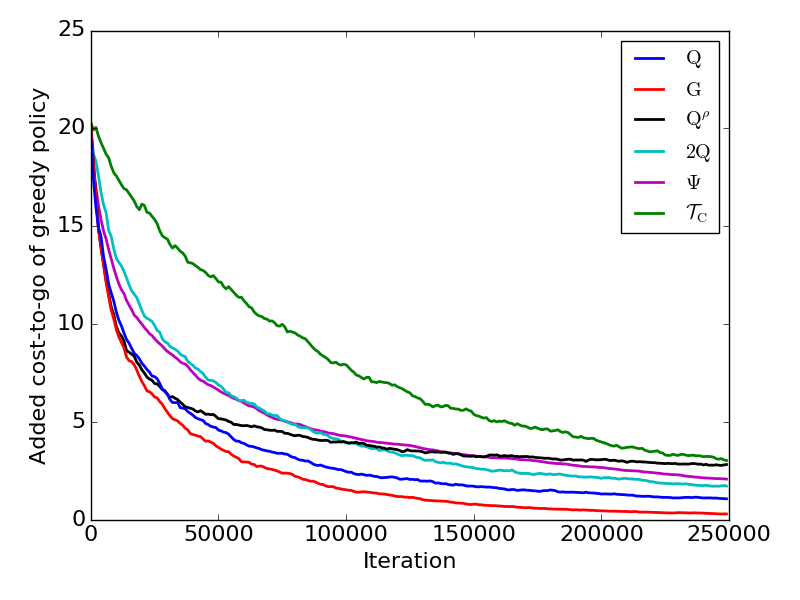}
    \end{subfigure}
    \\
    \begin{subfigure}[b]{0.32\textwidth}
        \includegraphics[width=\textwidth]{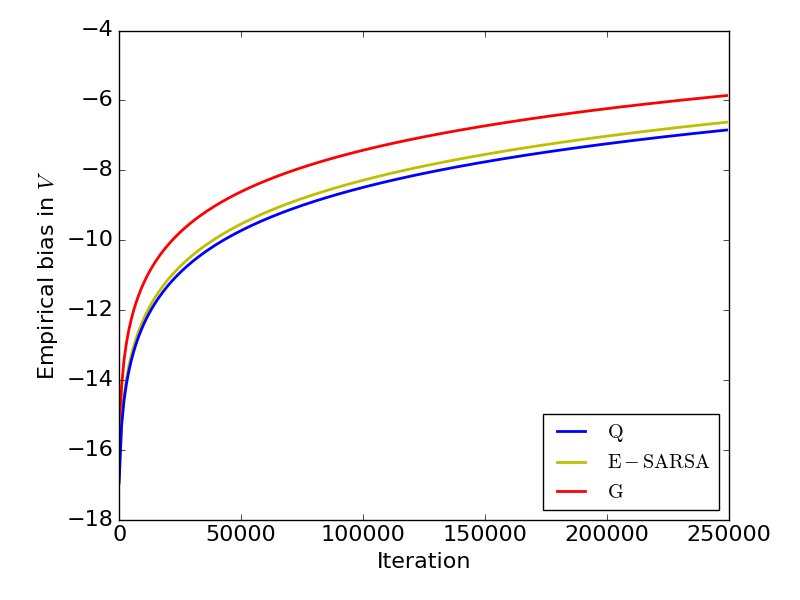}
    \end{subfigure}
    ~ 
    \begin{subfigure}[b]{0.32\textwidth}
        \includegraphics[width=\textwidth]{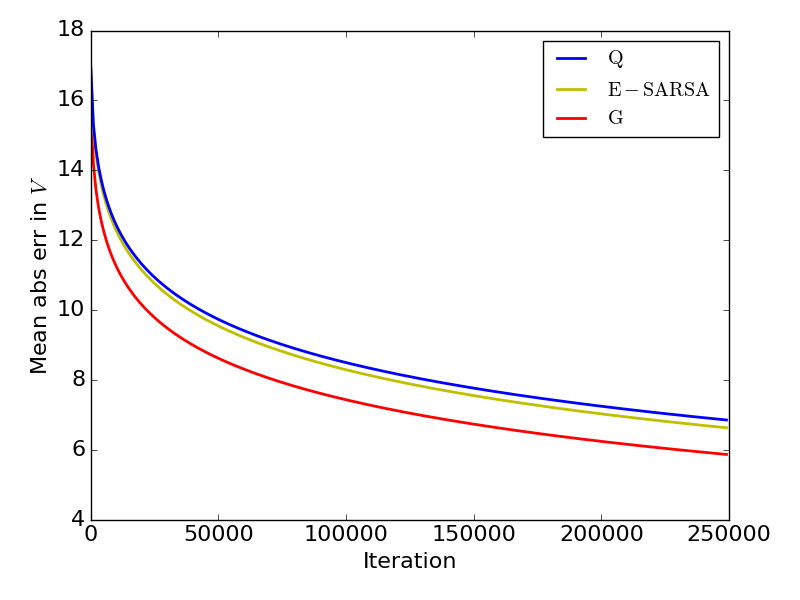}
    \end{subfigure}
    ~ 
    \begin{subfigure}[b]{0.32\textwidth}
        \includegraphics[width=\textwidth]{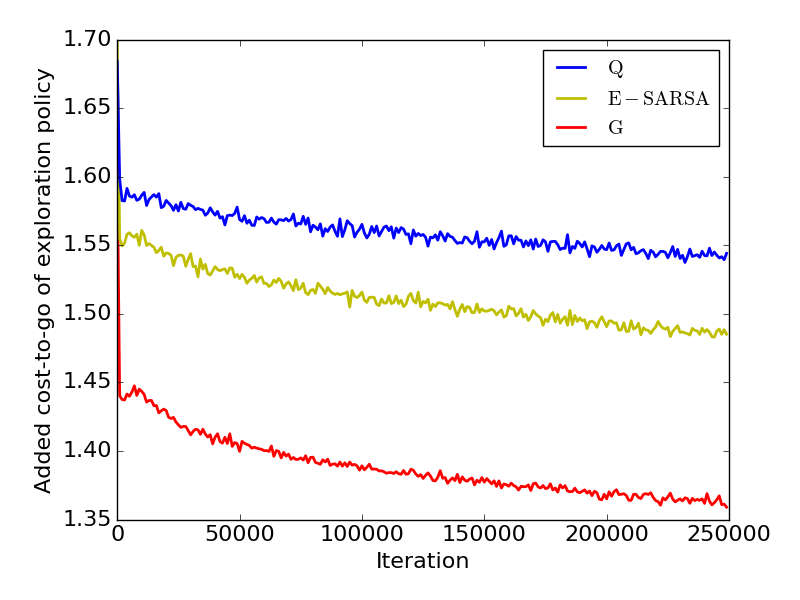}
    \end{subfigure}

     \caption{ 
     {\bf Gridworld (Rows 1-3):}
     Comparison of Q-, G-, Q$^\rho$-, Double-Q-, $\Psi$- and  $\C T_C$-learning.
     {\bf Row~1:} The cost in each step is fixed at 1.
     {\bf Row~2:} The cost in each step is distributed as $\C N(1,2^2)$.
     {\bf Row~3:} In each run, the domain is generated by drawing each $\E[c|s,a]$ uniformly 
     over $[1,3]$. The cost in each step is distributed as $\C N(\E[c|s,a],4^2)$. Note that in the noisy domains (Rows 2 and 3),
     G-learning dominates over all the other algorithms by the three measures. 
     {\bf Cliff~(Row~4):}
     Comparison of Q- and G-learning,  and Expected-SARSA. The cost in each step is 1, and falling off the cliff costs 5.
	{\bf Left:}~Empirical bias of $V$, relative to $V^*$~\eqref{eq:bias}. 
	{\bf Middle:}~Mean absolute error between $V$ and $V^*$~\eqref{eq:abserr}. 
	{\bf Right:}~Value of greedy policy, with the baseline $V^*$ subtracted~\eqref{eq:value}; except in Row 4, which shows the value of the exploration policy.
}

    \label{fig:res}
    
\end{figure*}

\subsection{CLIFF WALKING}
\label{sec:cliff}
Cliff walking is a standard example in reinforcement learning~\cite{sutton1998reinforcement}, 
that demonstrates an advantage of on-policy algorithms such as SARSA~\cite{sutton1998reinforcement,szepesvari2010algorithms} and Expected-SARSA~\cite{van2009theoretical,john1994best} over off-policy learning approaches such as Q-learning. 
We use it to show another interesting strength of  G-learning.

In this example, the agent can walk on the grid in Figure~\ref{fig:cliff} horizontally or vertically, with deterministic transitions.
Each step costs 1, except when the agent walks off the cliff (the bottom row), which costs 5, or reaches the goal (lower 
right corner), which costs 0.
In either of these cases, the position resets to the lower left corner.


\begin{figure}[!]

\centering
\includegraphics[width=0.3\textwidth]{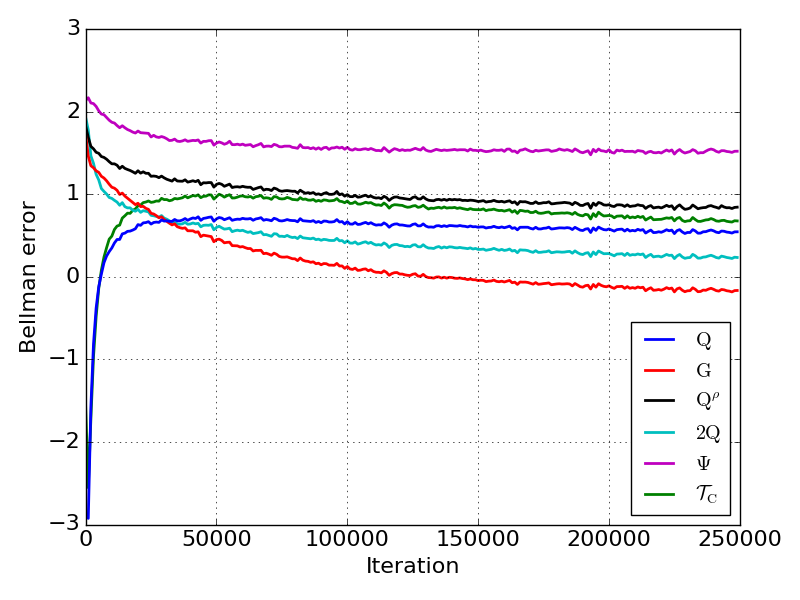}

\caption{ Running average of the Bellman error in the gridworld domain-generator example for 
Q-, G-, Q$^\rho$-, Double-Q-, $\Psi$- and  $\C T_C$-learning. The results for the other two gridworlds of Figure~\ref{fig:res} are similar.
}
\label{fig:bellman}
\end{figure}

Exploration is now on-line, with $s_t$ taken from the end of the previous step.
The exploration policy in our simulations is $\epsilon$-greedy with $\epsilon=0.1$,
i.e. with probability $\epsilon$ the agent chooses a random action, and otherwise it takes deterministically the one that seems optimal.
In practice, $\epsilon$ can be decreased after the learning phase, however it is also common to keep $\epsilon$ fixed for continued exploration~\cite{sutton1998reinforcement}.

In this setting, as shown in the bottom row of Figure~\ref{fig:res}, an off-policy algorithm like Q-learning performs poorly in terms of the value of its exploration policy, and the empirical cost it incurs.
It learns a rough estimate of $Q^*$ quickly, and then tends to use it and walk on the edge of the cliff.
This leads to the agent occasionally exploring the possibility of falling off the cliff.
In contrast, an on-policy algorithm like Expected-SARSA~\cite{van2009theoretical,john1994best} learns the value of its exploration policy, and quickly manages to avoid the cliff.

%
%

Figure~\ref{fig:cliff} compares Q-learning, G-learning and Expected-SARSA in this domain, and shows that G-learning learns to avoid the cliff even better than an on-policy algorithm, although for a different reason.
As an off-policy algorithm, G-learning does learn the value of the update policy, 
which prefers trajectories far from the cliff in the early stages of learning.
This occurs because near the cliff, avoiding the cost of falling requires ruling out downward moves, which has a high information cost.
On the other hand, trajectories far from the cliff, while paying a higher cost in  overall distance to the goal, enjoy lower information cost because acting randomly is not costly for them.

As shown in the bottom row of Figure~\ref{fig:res}, by using a greedy policy for $G$ as the basis of the $\epsilon$-greedy exploration, we enjoy the benefits of being aware of the value of the exploration policy during the learning stage. 
At the same time, G-learning converges faster than either Q-learning or Expected-SARSA to the correct value function.
In this case the ``noise'' that G-learning mitigates is related to the variability associated with the exploration.

\begin{figure}[!]
\centering
\includegraphics[width=0.3\textwidth]{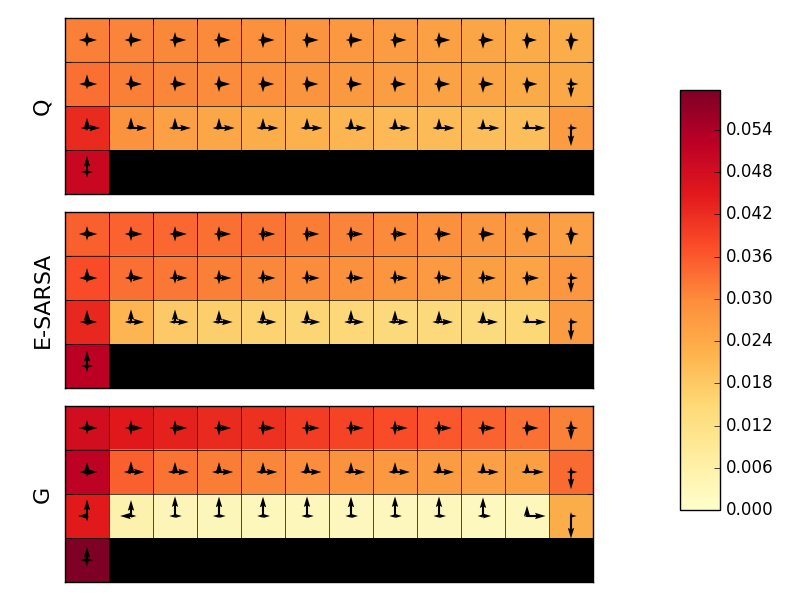}

\caption{ Cliff domain. The agent can choose a horizontally or vertically adjacent square, and moves there deterministically. The color scale and the arrow lengths indicate, respectively, the frequency of visiting each state and of making each transition, in the first 250,000 iterations of Q-learning, Expected-SARSA and G-learning. The near-greedy exploration policy of Q-learning has higher chance of taking the shortest path near the edge of the cliff at the bottom, than that of G-learning. As an off-policy algorithm, Q-learning fails to optimize for the exploration policy, whereas G-learning succeeds. }
\label{fig:cliff}
\end{figure}

\section{CONCLUSIONS}
\label{sec:conclusions}
The algorithm we have introduced successfully mitigates the slow learning problem of early stage Q-learning in noisy environments, 
that is caused by the bias generated by the hard optimization of the policy.

Although we have focused on Q-learning as a baseline, we believe that early-stage information penalties  
can  also be applied to advantage in more sophisticated model-free settings, such as TD($\lambda$),
and combined with other incremental learning techniques, such as function approximation, experience replay and actor-critic methods.

G-learning takes a Frequentist approach to estimating the optimal $Q$ function.
This is in contrast to Bayesian Q-learning~\cite{dearden1998bayesian}, which explicitly models the uncertainty about the $Q$ function as a posterior distribution.
It would be interesting to study the bias that hard optimization causes in the mean of this posterior, and to consider its reduction using methods similar to G-learning.

An important next step is to apply G-learning to more challenging domains, where an approximation of the $G$ function is necessary.
The simplicity of our linear $\beta$ schedule~\eqref{eq:sched} should facilitate such extensions, and allow G-learning to be combined with other schemes and algorithms.
Further study should also address the optimal schedule for $\beta$. We leave these important questions for future work.

\subsubsection*{Acknowledgments}
AP is supported by  ONR grant N00014-14-1-0243 and IARPA via DoI/IBC contract number D16PC00003.
RF and NT are supported by the DARPA MSEE Program, the Gatsby Charitable Foundation, the Israel Science Foundation and the Intel ICRI-CI Institute.

\small
\bibliographystyle{unsrt}

\bibliography{thebib}

\appendix

\section{CONVERGENCE OF G-LEARNING}\label{sec:apx}

In this section we prove the convergence of  $G$  to the optimal $G^*,$
with probability 1, under the G-learning update rule
\eq{
\MoveEqLeft G(s_t,a_t) \gets  (1-\alpha_t)G(s_t,a_t) 
\label{glearn2}
\\
\nn
& + \alpha_t \left( c_t - \tfrac{\gamma}{\beta} \log \left( \sum_{a'} \rho(a'|s_{t+1}) e^{-\beta G (s_{t+1},a')}  \right) \right).
} 
Recall that the supremum norm is defined as \mbox{$| x |_{\infty} = \max_{i}|x_i|$}, and that the optimal $G$ function
satisfies
\eq{
G^*(s,a) ={}&\E_\theta[c|s,a]
 \\
 \nn
 &-\tfrac\gamma\beta\E_p\left[\log\sum_{a'}\rho(a'|s')e^{-\beta G^*(s',a')}\right]
\\
\equiv{}& \Bb^*[G^*]_{(s,a)}.
\label{eq:hq2}
}
\setcounter{thm}{0}
The convergence proof relies on the following Lemma.
\begin{lem}
The operator $\Bb^*[G]_{(s,a)}$  defined in~\eqref{eq:hq2} is a contraction in the supremum norm.
\end{lem}
 \begin{proof}

Let us define 
\eq{ 
\label{bpg2}
\Bb^{\pi}[G]_{(s,a)}
={}&
 k^\pi(s,a) 
 \\
 \nn
& +  \gamma\sum_{s', a'} p(s'|s,a)  \pi(a'|s')  G(s',a') ,
}

where
\eq{ 
k^\pi(s,a) ={}&  \E_\theta[c|s,a] 
\\
& +  \tfrac{\gamma}{\beta} \sum_{s', a'} p(s'|s,a)  \pi(a'|s')  \log \tfrac{\pi(a'|s') }{\rho(a'|s')}.
\nn
}
Now, for any policy $\pi$, the operator~\eqref{bpg2} is a contraction  under the supremum norm~\cite{bertsekas1995dynamic}, i.e. for any $G_1$ and $G_2$
\eq{ 
| \Bb^{\pi}[G_1] - \Bb^{\pi}[G_2]  |_{\infty} \leq  \gamma  | G_1 - G_2  |_{\infty}.
}
Also note that
\eq{
\Bb^*[G_i]_{(s,a)} = \min_{\pi}  \Bb^{\pi}[G_i]_{(s,a)},
}
and that the optimum is achieved for 
\eq{
\pi_{G_i}(a|s) = \frac{ \rho(a|s)  e^{-\beta G_i(s,a)} }  {\sum_{a'} \rho(a'|s) e^{-\beta G_i(s,a')}   }    .
}
The Lemma now follows from 
\eq{
\MoveEqLeft  \big| \Bb^*[G_1] - \Bb^*[G_2]  \big|_{\infty}
  \\
\nn
& = \max_{(s,a)}  \left| \Bb^*[G_1]_{(s,a)} - \Bb^*[G_2]_{(s,a)}  \right|
\nn
\\
\nn
&=   \max_{(s,a)}  \left| \Bb^{\pi_{G_1}}[G_1]_{(s,a)} - \Bb^{\pi_{G_2}}[G_2]_{(s,a)}  \right|
\\
\intertext{(choose $i=\argmin\Bb^{\pi_{G_i}}[G_i]_{(s,a)}$)}
\nn
& \leq  \max_{(s,a)} \max_{i=1,2} \left| \Bb^{\pi_{G_i}}[G_1]_{(s,a)} - \Bb^{\pi_{G_i}}[G_2]_{(s,a)}  \right|
\\
\nn
& =   \max_{i=1,2}  \big| \Bb^{\pi_{G_i}}[G_1] - \Bb^{\pi_{G_i}}[G_2]  \big|_{\infty}
\\
\nn
&\leq  \gamma \big| G_1 - G_2 \big|_{\infty}.\qedhere
}
 \end{proof}

 The update equation~\eqref{glearn2} of the algorithm can be written as a stochastic iteration equation
 \eq{
 G_{t+1}(s_{t},a_{t}) ={}& (1-\alpha_t)G_t(s_t,a_t)
 \label{iter2}
 \\
 &+ \alpha_t ( \Bb^*[G_{t}]_{(s_t, a_t)} + z_t(c_t, s_{t+1}) )
 \nn
 }
 where the random variable $z_t$ is
 \eq{
 z_t(c_t, s_{t+1}) \equiv{}& -\Bb^*[G_{t}]_{(s_t, a_t)}
 \\
 &  + c_t - \tfrac{\gamma}{\beta} \log  \sum_{a'} \rho(a'|s_{t+1}) e^{-\beta G_t (s_{t+1},a') } .
 \nn
 }
Note that $z_t$ has expectation 0.  Many results exist for iterative equations of the type~\eqref{iter2}.
In particular, given conditions
\eq{
\sum_t\alpha_t=\infty;&&\sum_t\alpha_t^2<\infty,
\label{alphas2}
}
the contractive nature of $\Bb^*$,
infinite  visits to each pair $(s_t,a_t)$ and assuming that $|z_t|< \infty$ , 
$G_t$ is guaranteed to converge to the optimal $G^*$ with probability 1~\cite{bertsekas1995dynamic, borkar2008stochastic}. 

\end{document}